\documentclass{article}

\usepackage[preprint]{neurips_2026}

\usepackage[utf8]{inputenc}
\usepackage[T1]{fontenc}
\usepackage[
    colorlinks=true,
    linkcolor=blue,
    citecolor=blue,
    urlcolor=blue,
    filecolor=blue,
    bookmarks=true,
    bookmarksopen=true,
    bookmarksnumbered=true,
    pdfstartview=FitH,
    breaklinks=true
]{hyperref}
\usepackage{url}
\usepackage{booktabs}
\usepackage{microtype}
\usepackage{xcolor}
 
\usepackage{amsmath}
\usepackage{amssymb}
\usepackage{amsthm}
\usepackage{framed}
\usepackage{flafter}
\usepackage{subcaption}
\usepackage{bm}
\usepackage{pgfplots}
\pgfplotsset{compat=1.18}

\usepackage[ruled,lined,linesnumbered]{algorithm2e}
\usepackage[capitalize,noabbrev]{cleveref}
\theoremstyle{plain}
\newtheorem{theorem}{Theorem}[section]
\newtheorem{lemma}[theorem]{Lemma}

\theoremstyle{definition}
\newtheorem{definition}[theorem]{Definition}
\newtheorem{assumption}[theorem]{Assumption}

\crefname{assumption}{assumption}{assumptions}
\Crefname{assumption}{Assumption}{Assumptions}

\theoremstyle{remark}
\newtheorem{remark}[theorem]{Remark}

\newcommand{\argmin}{\mathop{\mathrm{arg\,min}}\limits}
\newcommand{\argmax}{\mathop{\mathrm{arg\,max}}\limits}
\newcommand{\cc}{\mathrm{c}}
\newcommand{\dd}{\mathrm{d}}

\title{Best Arm Identification in Generalized Linear Bandits via Hybrid Feedback}

\author{%
{\bfseries
Qirun Zeng$^{1}$ \quad
Xuchuang Wang$^{2}$ \quad
Jiayi Shen$^{3}$} \\[0.2em]
{\bfseries
Xutong Liu$^{4}$ \quad
Fang Kong$^{5}$ \quad
Jinhang Zuo$^{1}$\thanks{Corresponding author.}} \\[0.7em]
{\small
$^{1}$City University of Hong Kong, Hong Kong SAR, China} \\[-0.05em]
{\small
$^{2}$Hong Kong Baptist University, Hong Kong SAR, China} \\[-0.05em]
{\small
$^{3}$University of Science and Technology of China, Hefei, China} \\[-0.05em]
{\small
$^{4}$University of Washington Tacoma, Tacoma, WA, USA} \\[-0.05em]
{\small
$^{5}$Southern University of Science and Technology, Shenzhen, China} \\[0.4em]
{\small
\texttt{qirun.zeng@my.cityu.edu.hk}
\quad
\texttt{xuchuangw@hkbu.edu.hk}} \\[-0.05em]
{\small
\texttt{jiayishen@mail.ustc.edu.cn}
\quad
\texttt{xutongl@uw.edu}} \\[-0.05em]
{\small
\texttt{kongf@sustech.edu.cn}
\quad
\texttt{jinhang.zuo@cityu.edu.hk}}
}

\usepackage{perpage}
\MakePerPage{footnote}

\begin{document}

\maketitle

\begin{abstract}
We study fixed-confidence best arm identification in generalized linear bandits under a hybrid feedback model: at each round, the learner may query either (i) absolute reward feedback from a single arm or (ii) relative (dueling) feedback from an arm pair, both governed by generalized linear models. We introduce a likelihood-ratio--based confidence sequence that unifies heterogeneous generalized linear observations and yields an explicit ellipsoidal confidence set under a self-concordance assumption. Building on this confidence set, we propose a hybrid Track-and-Stop algorithm that adaptively allocates queries by tracking a minimax-optimal design over a joint action space of arms and pairs. We establish $\delta$-correctness and provide almost-sure asymptotic upper bounds on the stopping time. We further extend the framework to a cost-aware setting that accounts for heterogeneous acquisition costs across feedback modalities. Empirical experiments demonstrate that the proposed algorithms significantly improve sample efficiency over baseline methods.
\end{abstract}

\section{Introduction}

Large-scale interactive learning systems increasingly rely on heterogeneous user feedback.
In modern conversational and recommendation platforms, especially LLM assistants trained with reinforcement learning from human feedback (RLHF)~\citep{christiano2017deep,ouyang2022training}, supervision often comes from pairwise preferences, while absolute signals such as ratings, clicks, or task completion indicators may also be available.
Although these feedback modalities may arise from a shared latent utility structure, they provide fundamentally different statistical information: absolute feedback offers noisy evaluations of individual actions, whereas relative feedback captures pairwise comparisons of their underlying utilities.
Designing learning algorithms that can systematically exploit such heterogeneous yet complementary feedback remains a central challenge in interactive learning.

From a theoretical perspective, existing bandit models mostly treat different feedback modalities in isolation.
Stochastic and linear bandits focus exclusively on reward-based observations, while dueling bandits and preference-based learning frameworks assume access only to pairwise comparisons~\mbox{\citep{yue2011beat,yue2012k,sui2018advancements}}. 
Recent work has started to bridge this gap by studying hybrid feedback bandits~\citep{wang2025fusing}.
However, existing approaches mainly combine reward and dueling feedback at the algorithmic level via joint elimination procedures. 
As a result, they do not fully exploit the shared latent structure underlying the two modalities within a unified statistical inference framework.

A natural way to move from algorithmic fusion to statistical integration is to model absolute and relative feedback as different observations of a shared latent utility. Generalized linear bandits (GLBs) provide such a framework: absolute observations provide noisy nonlinear feedback on individual arm utilities, while dueling observations provide noisy nonlinear feedback on utility differences. Thus, both modalities provide complementary information about the same latent parameter through distinct observation channels.
This motivates our study of fixed-confidence best-arm identification (BAI) with hybrid feedback under generalized linear structure, where the goal is to identify the best arm with high confidence while adaptively choosing between absolute and relative queries.

Addressing this problem requires more than simply combining two feedback channels, and introduces several new challenges.
First, reward and dueling feedback carry heterogeneous information: reward queries observe individual arm utilities, while dueling queries observe utility differences.
Their distinct nonlinear likelihoods and curvature profiles make naive aggregation statistically inefficient.
Second, fixed-confidence inference is harder under generalized linear structure, since uncertainty depends on both the unknown latent parameter and the queried arm or arm pair.
This complicates the construction of time-uniform confidence regions for adaptive sampling and stopping.
Third, exploration must be optimized jointly over actions and feedback modalities.
The learner must decide which arms or arm pairs to query, when to use absolute or relative feedback, and how to account for potentially heterogeneous acquisition costs.
Together, these challenges call for a unified inference, stopping, and allocation framework that estimates the shared latent utility, certifies the best arm, and adaptively balances absolute and relative feedback during pure exploration.

\begin{table*}[t]
  \small
  \centering
\caption{Comparison of fixed-confidence BAI across feedback models and structural assumptions.}
  \label{tab:bai}
  \begin{tabular}{lcccc}
    \toprule
    & Model & $C_{\text{reward}}$ & $C_{\text{dueling}}$$^*$ & Sample Complexity$^\dagger$ \\
    \midrule
    \citet{jedra2020optimal} & Linear & 1 & $\emptyset$  & $\tilde{\mathcal O}\big(T^\star_{1,\emptyset}(\bm\theta^\star)\log\frac{1}{\delta}\big)$  \\
    \citet{jun2021improved} & Logistic & 1 ($\emptyset$) & $\emptyset$ (1) & $\tilde{\mathcal O}\left(T'_{1(\emptyset), \emptyset(1)}(\bm\theta^\star)\log\frac{1}{\delta}\right)$\\
    \citet{rivera2024near} & Logistic & 1 & $\emptyset$ & $\tilde{\mathcal O} \left(T^{\star}_{1, \emptyset}(\bm \theta^\star) \log^2\frac{1}{\delta}\right)$\\
    \citet{kazerouni2021best} & GLM & 1 & $\emptyset$ & $\tilde{\mathcal{O}} \left(T''_{1,\emptyset}(\bm \theta^\star) \log^2 \frac{1}{\delta}\right)$\\
    HyTS (Sec.~\ref{sec:g-optimal}) & Hybrid GLM & 1 & 1 & $\tilde{\mathcal O}\big(T^\star_{1, 1}(\bm\theta^\star)\log\frac{1}{\delta}\big)$\\
    Cost-aware-HyTS (Sec.~\ref{sec:cost-aware-ts}) & Hybrid GLM & $a$ & $b$ & $\tilde{\mathcal O}\big(T^\star_{a,b}(\bm\theta^\star)\log\frac{1}{\delta}\big)$\\
    \bottomrule
  \end{tabular}\\{

\begin{minipage}{\textwidth}
\footnotesize
$^*$$C_{\text{reward}}$ and $C_{\text{dueling}}$ denote the acquisition cost of reward and dueling feedback, respectively, with $\emptyset$ indicating unavailable.\\
$^\dagger$ The sample complexity is expressed in terms of model-specific characteristic times or complexity measures. For our algorithms, $T^\star$ denotes the characteristic time induced by the minimax-optimal design under the corresponding feedback costs (Section~\ref{sec:g-optimal}). $T'$ and $T''$ denote the finite-time complexity measures of \citet{jun2021improved} and \citet{kazerouni2021best}, respectively.
\end{minipage}
  }
  \vspace{-10pt}
\end{table*}

\subsection{Our Contributions}

We study fixed-confidence BAI in generalized linear bandits under a hybrid feedback model that allows both reward and dueling queries.
Table~\ref{tab:bai} summarizes representative BAI results and highlights the gap addressed by this work.
Our main contributions are summarized as follows.

\begin{itemize}
\item
We formulate structured pure exploration with \emph{both reward and dueling feedback} under a shared GLM parameterization.
Unlike unstructured hybrid-bandit models~\citep{wang2025fusing}, this setting transfers information through a latent parameter while allowing the two modalities to have different likelihoods and curvatures.

\item
We extend likelihood-ratio--based confidence sequences for self-concordant GLMs to hybrid observations.
Despite the heterogeneous likelihoods induced by reward and dueling feedback, the resulting confidence region admits a single explicit ellipsoidal form, enabling unified inference and fixed-confidence stopping.

\item
Building on these confidence sets, we design HyTS-GLB, a Track-and-Stop--style algorithm that tracks a minimax design over the joint query space of arms and arm pairs.
We prove $\delta$-correctness and almost-sure asymptotic upper bounds on the stopping time as $\delta \to 0$.

\item
We extend HyTS-GLB to \emph{heterogeneous acquisition costs} and provide an almost-sure asymptotic upper bound on total cost as $\delta\to0$.
This lets the learner trade off both information and cost across reward and dueling queries.

\item
We conduct experiments showing that hybrid feedback reduces sample complexity relative to single-modality baselines, while the cost-aware variant further adapts to asymmetric query costs across modalities.
\end{itemize}

Together, these contributions yield a unified statistical and algorithmic framework for fixed-confidence best-arm identification with hybrid feedback in generalized linear bandits.
All detailed proofs are deferred to the appendix due to space limitations.

\subsection{Related Work}

\textbf{Generalized Linear Bandits.}
GLMs~\citep{mccullagh2019generalized} model exponential-family observations with natural parameter $\bm{x}^\top\bm\theta^\star$ and mean $\mu(\bm{x}^\top\bm\theta^\star)$.
\citet{filippi2010parametric} introduced GLM-based structured bandits, followed by work on regret and inference under assumptions such as self-concordance~\citep{russac2021self}.
\citet{lee2024unified} developed likelihood-ratio--based confidence sequences for self-concordant GLMs, yielding convex and numerically tight confidence regions.
We go beyond \citet{lee2024unified} by developing confidence sequences for hybrid reward–dueling feedback, where adaptive queries may arise from different GLM channels. This requires a single confidence certificate that aggregates heterogeneous Fisher curvature and feature geometry across modalities, a setting absent in prior single-channel GLM analyses.
Beyond that, \citet{kirschner2025confidence,clerico2025confidence} also construct similar confidence sequences for GLMs.

\textbf{Best-Arm Identification and Track-and-Stop.}
Best-arm identification has been widely studied for stochastic and linear bandits~\citep{bubeck2009pure,audibert2010best,soare2014best,jourdan2022choosing,shao2025linear}.
\citet{garivier2016optimal} established a general fixed-confidence lower bound for multi-armed bandits and proposed Track-and-Stop, which asymptotically achieves it.
For linear bandits, \citet{soare2015sequential} derived an instance-dependent lower bound, and \citet{jedra2020optimal} gave a matching algorithm.
\citet{degenne2020gamification} showed that iterative saddle-point solvers can design asymptotically optimal BAI algorithms for structured bandits.
For GLBs, \citet{kazerouni2021best} studied $(\epsilon,\delta)$-PAC BAI and provided corresponding sample-complexity guarantees.

\textbf{Relative and Hybrid Feedback.}
Relative (dueling) feedback has been studied extensively in bandit learning, particularly in dueling bandits and preference-based learning~\citep{yue2011beat,yue2012k,sui2018advancements}.
It is also practically attractive because relative feedback can be easier to elicit than absolute rewards in conversational recommendation~\citep{zhang2020conversational,yang2024conversational,zou2026conversational}.
Most closely related, \citet{wang2025fusing} study stochastic bandits with reward and dueling feedback and propose regret-minimization algorithms based on elimination and decomposition, showing that hybrid feedback helps even without structure. 
We instead study fixed-confidence BAI under shared GLM structure, which requires time-uniform confidence sequences, confidence-based stopping, and design-based allocation over arms and arm pairs.
Related hybrid settings have also been explored by \citet{he2024learning} for multi-armed bandits, dueling bandits, and bandits with offline data.

\section{Problem Formulation}
\label{sec:problem-formulation}

\textbf{Generalized Linear Model.}
We begin by introducing the generalized linear model~\citep{mccullagh2019generalized} underlying both reward and dueling feedback. Let $r \in \mathbb{R}$ denote a generic scalar observation. In a GLM, the conditional distribution of $r$ given a scalar parameter $\eta \in \mathbb{R}$ belongs to an exponential family and is specified by
\begin{equation}
p(r \mid \eta)
=
\exp\!\left(
\frac{r \eta - b(\eta)}{\zeta(\varphi)} + c(r,\varphi)
\right),
\label{eq:pdf}
\end{equation}
where $\eta$ is the canonical (natural) parameter, $\varphi$ is a dispersion parameter, $b(\cdot)$ is the log-partition function, and $c(r,\varphi)$ is the base measure. We use $p(r \mid \eta)$ to denote the conditional probability density (or mass) function of $r$ given $\eta$.
A fundamental property of GLMs is that the conditional mean of $r$ is given by
$\mathbb{E}[r \mid \eta] = b'(\eta) \triangleq \mu(\eta)$,
where $\mu(\cdot)$ is the mean function (inverse link function). Many common observation models, including Gaussian, Poisson, and Bernoulli distributions, fall within this framework.

\textbf{Hybrid Generalized Linear Bandits.}
We consider a stochastic decision-making problem with $K$ arms indexed by $\mathcal K \triangleq \{1,2,\dots,K\}$. 
Each arm $i \in \mathcal K$ is associated with a known feature vector $\bm{x}_i \in \mathbb{R}^d$, and the environment is governed by an unknown parameter $\bm{\theta}^\star \in \mathbb{R}^d$. We assume the set of arms $\mathcal K$ spans $\mathbb R^d$.
We also assume that the dueling feature vectors span $\mathbb R^d$, i.e., $\operatorname{span}\{\bm x_i-\bm x_j:1\le i<j\le K\}=\mathbb R^d$.
The learner interacts with the environment sequentially and may acquire information through two different feedback modalities.

At each round $t$, the learner chooses either a \emph{reward query} or a \emph{dueling query}.
In the former, it selects $i_t \in \mathcal K$ and observes
$R_t \sim p_{\cc}(\cdot \mid \bm{x}_{i_t}^\top \bm{\theta}^\star)$.
In the latter, it selects $(j_t,k_t) \in \mathcal G \triangleq \{(j,k)\in\mathcal K^2: j<k\}$ and observes
$Y_t\in\{0,1\}$ drawn from
$p_{\dd}(\cdot \mid \bm{x}_{j_t,k_t}^\top \bm{\theta}^\star)$,
where $\bm{x}_{j,k}\triangleq \bm{x}_j-\bm{x}_k$, and $Y_t=1$ indicates preference for $j_t$.

Both feedback modalities are modeled as GLMs with a shared parameter $\bm{\theta}^\star$, but possibly different observation spaces and mean functions $\mu_{\cc}$ and $\mu_{\dd}$.
We use subscripts $\cc$ and $\dd$ to distinguish reward and dueling quantities.
Define the joint action space $\mathcal A \triangleq \mathcal K \cup \mathcal G$ and, for $a=(j,k)\in\mathcal G$, write $\bm x_a=\bm x_{j,k}$.
For each $a\in\mathcal A$, we define the feedback modality $m(a)$ to be $\cc$ for $a\in\mathcal K$ and $\dd$ for $a\in\mathcal G$.
Selecting $a_t \in \mathcal A$ corresponds to querying a single arm ($a_t \in \mathcal K$) or a pair ($a_t \in \mathcal G$), with observations generated by the corresponding model $p_{m(a_t)}(\cdot\mid \bm x_{a_t}^\top\bm\theta^\star)$.

\begin{remark}
The shared parameter $\bm\theta^\star$ represents a latent user- or evaluator-preference vector, which is natural in recommender systems and RLHF-style data collection (e.g., \citet{yue2012k,negahban2012iterative,christiano2017deep,yang2024conversational}). An item or response may receive an absolute rating through $\mu_{\cc}(\bm x_i^\top\bm\theta^\star)$, while a pairwise comparison is generated from the utility difference through $\mu_{\dd}((\bm x_j-\bm x_k)^\top\bm\theta^\star)$. Thus, the two modalities may have different links and noise levels while still providing complementary observations of the same underlying preference structure. This flexibility is not merely cosmetic: choosing different link functions for reward and dueling feedback can change the local curvature and information carried by each query, and may yield gains in regimes where one modality is saturated, noisier, or better aligned with the relevant decision boundary.
\end{remark}

\textbf{Best Arm Identification Objective.}
Under the hybrid generalized linear bandit problem described above, we study a fixed-confidence best arm identification (BAI) problem. 
The learner sequentially queries reward or dueling feedback and aims to identify the single arm with the highest expected reward using as few samples as possible.
Formally, the optimal arm is defined as \[
    i^\star
    \triangleq
    \argmax_{i \in \mathcal K} \mu_{\cc}\!\left(\bm{x}_i^\top \bm{\theta}^\star\right),
\]
and we assume it is unique. 
We consider adaptive strategies that select both the feedback modality and the queried arms over time. 
Let $\tau$ denote a stopping time and let $\hat i^\star_\tau$ be the arm recommended upon termination.

\begin{definition}
Given a confidence level $\delta \in (0,1)$, an exploration strategy is said to be \emph{$\delta$-correct} if, for every $\bm\theta\in\Theta$ with a unique best arm $i^\star(\bm\theta)$,
\[
    \Pr_{\bm\theta} (\hat i^\star_\tau \neq i^\star(\bm\theta)) \le \delta
    \quad\text{and}\quad
    \Pr_{\bm\theta} (\tau < +\infty) = 1.
\]
\end{definition}

The goal is to design a $\delta$-correct exploration strategy with a small stopping time.
We adopt the following standard assumptions from the generalized linear bandit literature~\citep{russac2021self,lee2024unified}.

\begin{assumption}[Bounded parameter]\label{ass:theta-bound}
    The unknown parameter satisfies $\bm{\theta}^\star \in \Theta \subseteq \mathcal B^d(S)
    \triangleq \{\bm{\theta} \in \mathbb R^d : \|\bm{\theta}\|_2 \le S\}$,
    where $\Theta$ is compact and convex.
\end{assumption}

\begin{assumption}[Bounded features]\label{ass:arm-bound}
    The arm feature set satisfies $\mathcal X \triangleq \{\bm{x}_i : i \in \mathcal K\} \subseteq \mathcal B^d(1)$.
    This implies $\|\bm{x}_{j,k}\|_2 \le 2$ for all $(j,k) \in \mathcal G$.
\end{assumption}

\begin{assumption}[Smoothness and convexity]\label{ass:diff-convex}
The log-partition functions $b_{\cc}(\cdot)$ and $b_{\dd}(\cdot)$ are three times continuously differentiable and convex, with $\mu_{m(a)}'(\cdot) = b_{m(a)}''(\cdot) \ge 0$ for every $a\in\mathcal A$.
\end{assumption}

Bounded parameters and features control the predictors and likelihood derivatives. Compactness of $\Theta$ ensures MLE existence in \Cref{sec:MLE}, while convexity makes its computation a convex problem.
Monotonicity of $\mu_{\cc}$ and uniqueness of the best arm imply $i^\star={\argmax}_{i\in\mathcal K}\bm x_i^\top \bm \theta^\star$.

\section{Confidence Sequences for Hybrid GLB}\label{sec:cs-hybrid}

In this section, we develop time-uniform confidence sequences (CS) for $\bm{\theta}^\star$ under the hybrid feedback model.
Our construction follows the likelihood-ratio framework of \citet{lee2024unified}, adapted to mixed feedback.
We first build a likelihood-ratio--based sequence, then use self-concordance to derive an explicit ellipsoidal set for algorithmic design.

\subsection{Maximum Likelihood Estimation}
\label{sec:MLE}

To distinguish between reward and dueling feedback, we define the following index sets according to the action $a_s$ selected at round $s$: $\mathcal C_t \triangleq \{ 1\le s < t : a_s \in \mathcal K\}$ and $\mathcal D_t \triangleq \{ 1\le s < t : a_s \in \mathcal G\}$.
Let $Z_s$ denote the observation at round $s$, i.e., $Z_s=R_s$ if $a_s\in\mathcal K$ and $Z_s=Y_s$ if $a_s\in\mathcal G$.

Under the GLM in \Cref{eq:pdf}, the cumulative negative log-likelihood at time $t$ can be decomposed as \[
    \mathcal L_t(\bm \theta) \triangleq \sum_{s=1}^{t-1} \frac{b_{m(a_s)}(\bm{x}_{a_s}^\top \bm\theta) - Z_s \bm x_{a_s}^\top \bm \theta}{\zeta_{m(a_s)}(\varphi_{m(a_s)})}.
\]
We define the $\ell_2$-constrained maximum likelihood estimator (MLE) as
\begin{equation}
    \hat{\bm{\theta}}_t
    \triangleq
    \argmin_{\bm \theta \in \Theta} \mathcal L_t(\bm\theta),
    \label{eq:theta-mle}
\end{equation}
where $\Theta \subset \mathbb R^d$ is a compact, convex parameter set. Equivalently, whenever $\hat{\bm\theta}_t$ lies in the interior of $\Theta$, it satisfies the first-order optimality condition $\nabla \mathcal L_t(\hat{\bm\theta}_t) = 0$.

\subsection{Likelihood-Ratio Confidence Sequences}

We first construct a likelihood-ratio confidence sequence based on the cumulative loss.
Let $L_t$ be any almost-sure upper bound on the Lipschitz modulus of $\mathcal L_t$ over $\Theta$:
\[
L_t \ge \sup_{\bm{\theta},\bm{\theta}' \in \Theta}
\frac{|\mathcal L_t(\bm{\theta}) - \mathcal L_t(\bm{\theta}')|}{\|\bm{\theta} - \bm{\theta}'\|_2}.
\]
Following \citet{lee2024unified}, we use an upper bound on the Lipschitz modulus that holds almost surely.
In the hybrid model, $L_t$ decomposes over reward and dueling observations and thus scales with their sample counts.
Since $\|\bm{x}_{i,j}\|_2 \le 2$ while $\|\bm{x}_i\|_2 \le 1$, the dueling term incurs an additional factor of $2$.

\begin{lemma}[Unified likelihood-ratio CS for hybrid GLMs, Theorem~3.1 of~\citet{lee2024unified}]
\label{lem:cs-hybrid}
Suppose Assumptions~\ref{ass:theta-bound}--\ref{ass:diff-convex} hold. Then for any $\delta \in (0,1)$,
    \[
        \Pr \big( \exists\, t \ge 1 : \mathcal L_t(\bm \theta^\star) - \mathcal L_t(\hat{\bm\theta}_t)
            \ge \beta_t(\delta) \big)
        \le \delta,
    \]
    with a time-uniform confidence radius:
    \begin{equation}
        \beta_t(\delta)
        \triangleq
        \log \tfrac{1}{\delta}
        + \inf_{c_t \in (0,1]}
        \left\{
            d \log\!\tfrac{1}{c_t}
            + 2 S L_t c_t
        \right\} \le
        \log \tfrac{1}{\delta}
        + d \log\!\left(
            \max\!\left\{ e,\; \tfrac{2 e S L_t}{d} \right\}
        \right).
        \label{eq:beta-hybrid}
    \end{equation}
\end{lemma}
The bound constructed in Appendix~\ref{app:proof-hybrid-cs} satisfies $L_t=O(t)$ almost surely, and hence $\beta_t(\delta)=O(\log(1/\delta)+d\log t)$ almost surely as $t\to\infty$.

\subsection{Ellipsoidal Confidence Sequences under Self-Concordance}

The likelihood-ratio CS controls the loss difference but is not directly helpful to algorithmic design. We therefore exploit self-concordance to convert it into an ellipsoidal confidence sequence.

\begin{assumption}[Self-concordant GLMs~\citep{russac2021self}]
\label{ass:self-concordant}
    The GLMs of reward and dueling feedback are self-concordant. That is, there exist finite constants $M_\cc$ and $M_{\dd}$ such that, for all $a\in\mathcal A$,
        $|\mu_{m(a)}''(\eta)| \le M_{m(a)} \mu_{m(a)}'(\eta),
        \forall\,
        \eta \in \{\bm{x}_a^\top \bm{\theta}:\bm{\theta}\in\Theta\}$.
\end{assumption}

This assumption enables the ellipsoidal conversion in \Cref{thm:ellip-hybrid}.
Although $\hat{\bm{\theta}}_t$ minimizes the combined loss $\mathcal L_t$ rather than each component, self-concordance holds uniformly over $\Theta$ and does not require it to minimize $\mathcal L_t^\cc$ or $\mathcal L_t^\dd$.
This is what lets the two feedback streams share one estimator and one confidence certificate rather than producing two separate ellipsoids that would have to be reconciled afterward.
To unify the two feedback modalities in inference and design, we define the aggregated information matrix
    \[
    \bm A_t
    \triangleq
    \sum_{s=1}^{t-1}
    \frac{
        \mu_{m(a_s)}'(\bm x_{a_s}^\top\hat{\bm\theta}_t)
    }{
        2(1+S\rho_{a_s}M_{m(a_s)})\zeta_{m(a_s)}(\varphi_{m(a_s)})
    }
    \bm x_{a_s}\bm x_{a_s}^\top,
    \]
where $\rho_a=1$ for $a\in\mathcal K$ and $\rho_a=2$ for $a\in\mathcal G$.
Intuitively, the matrix $\bm A_t$ aggregates curvature information from both feedback types and serves as the effective information matrix in the ellipsoidal confidence set.

\textit{Identifiability.}
Throughout, we impose the identifiability condition (\citet{kazerouni2021best}): there exists a (random but finite) time $E$ such that the matrix $\bm{A}_E$ is positive definite. Consequently, $\bm A_t \succ 0$ for all $t \ge E$, and $\bm A_t^{-1}$ is well defined.

Despite heterogeneous curvature induced by reward and dueling observations, their contributions combine additively into a single information matrix, yielding a unified ellipsoidal confidence region.

\begin{theorem}[CS for hybrid GLMs]
\label{thm:ellip-hybrid}
    Define
        $\mathcal E_t(\delta)
        \triangleq
        \{
            \bm{\theta} \in \mathbb R^d :
            \|\bm{\theta} - \hat{\bm{\theta}}_t\|_{\bm{A}_t}^2
            \le \beta_t(\delta)
        \}$.
    We have
    \[
        \Pr \left( \exists\, t \ge 1 : \bm{\theta}^\star \notin \mathcal E_t(\delta) \right)
        \le \delta.
    \]
\end{theorem}
This ellipsoidal form enables a simple stopping rule and a design-based allocation strategy.

\begin{proof}[Proof sketch]
It suffices to show $\|\bm{\theta}-\hat{\bm{\theta}}_t\|^2_{A_t} \le \mathcal L_t(\bm{\theta})-\mathcal L_t(\hat{\bm{\theta}}_t)$.
By \citet[Lemma~D.1]{lee2024unified} and Assumption~\ref{ass:self-concordant}, each component loss admits the corresponding quadratic lower bound at $\hat{\bm{\theta}}_t$; summing yields $\mathcal L_t(\bm{\theta})-\mathcal L_t(\hat{\bm{\theta}}_t) \ge \|\bm{\theta}-\hat{\bm{\theta}}_t\|_{A_t}^2$, which with~\Cref{lem:cs-hybrid} concludes the proof.
\end{proof}

\textbf{Relation to~\citet{lee2024unified}.}
Our construction adapts the likelihood-ratio analysis to reward and dueling channels with potentially different mean functions and dispersion scales.
The matrix $\bm A_t$ combines their Hessian contributions according to the realized queries and modality-specific self-concordance factors, preserving each channel's curvature and feature geometry.
The resulting ellipsoid provides a common confidence certificate for the stopping rule and guides query allocation across the two feedback channels.

\section{Hybrid Track-and-Stop Algorithms}
\label{sec:sampling-strategy}

This section introduces a hybrid sampling framework that integrates reward and dueling feedback within a unified pure-exploration procedure. The central idea is to exploit the complementary statistical information provided by the two feedback modalities, while adaptively allocating samples across them. Our approach extends the Track-and-Stop paradigm~\citep{garivier2016optimal,jedra2020optimal} to hybrid generalized linear bandits.

\subsection{HyTS-GLB Algorithm}
\label{sec:g-optimal}

\textbf{Minimax-Optimal Design.}
Following the Track-and-Stop principle, we define sampling proportions through a minimax optimal-design problem. The criterion allocates samples so as to minimize the largest confidence width among all comparisons between the current best arm and its competitors. It is closely related to classical $G$-optimal design~\citep{kiefer1960equivalence,pukelsheim2006optimal}, which has played a central role in best-arm identification for linear bandits~\citep{soare2014best}.

To leverage the ellipsoidal confidence sets for efficient pure exploration, we now specify how samples should be allocated across actions. Let $\hat{\bm\theta}_t$ denote the maximum likelihood estimator at time $t$, and define the current empirical best arm, with ties broken arbitrarily, as
    $\hat i_t^\star = {\argmax}_{i\in\mathcal{K}} \bm{x}_i^\top \hat{\bm\theta}_t $.
    
Given a sampling weight vector $\bm{w} \in \Delta_{\mathcal{A}} \triangleq \{\bm{w} \mid  w_a \ge 0 \land \sum_{a\in\mathcal A} w_a = 1\}$, we define the hybrid information matrix
    \[
    \bm A(\bm w,\bm\theta)
    \triangleq
    \sum_{a\in\mathcal A}
    \bm w_a \,
    \frac{\mu_{m(a)}'(\bm{x}_a^\top \bm\theta)}
    {2(1 + S\rho_a M_{m(a)})\zeta_{m(a)}(\varphi_{m(a)})}
    \, \bm{x}_a \bm{x}_a^\top.
    \]
The minimax-optimal sampling proportions $\bm w^\star(t)$ are any solution to
\[
    \bm w^\star (t) =
    \argmin_{\bm w \in \Delta_{\mathcal{A}}}
    \phi ({\hat{\bm\theta}_t}, \bm w) \triangleq \begin{cases} 
        \max_{i \ne \hat i^\star_t} 
        \left\| \bm{x}_{\hat i^\star_t, i} \right\|^2_{\bm A(\bm w, {\hat{\bm\theta}_t})^{-1}}, &\text{if } \bm A(\bm w, \hat{\bm \theta}_t) \succ 0,\\
        +\infty, &\text{otherwise}.
    \end{cases}
\]
Since they specify asymptotic sampling frequencies rather than an explicit sequential policy, we adopt a tracking rule to realize these proportions online.

\textbf{Tracking Rule.}
The learner tracks the target proportions by selecting the most under-sampled action relative to its cumulative target mass, while incorporating a vanishing forced-exploration mechanism to guarantee persistent excitation.
Let $\mathcal A_{0} \subset \mathcal A$ be a fixed subset of arms such that
\(
\lambda_{\min}\!\bigl(\sum_{a \in \mathcal A_0} \bm x_{a} \bm x_{a}^\top\bigr) > 0,
\) where $\lambda_{\min}(\cdot)$ denotes the smallest eigenvalue, and let $\pi_{0}$ denote a probability distribution supported on $\mathcal A_0$ with $\pi_0(a)>0$ for every $a\in\mathcal A_0$ (e.g., the uniform distribution).
Let $\mathcal T_t$ be the set of tracking rounds up to time $t$, and let $N^{\mathcal T}_a(t)=\sum_{s\in\mathcal T_t}\mathbf 1\{a_s=a\}$ and $W^{\mathcal T}_a(t)=\sum_{s\in\mathcal T_t}w_a^\star(s-1)$ denote the corresponding pull counts and cumulative target mass.
At each round,
\begin{equation}
a_t \in 
\begin{cases}
\argmin_{a \in \mathcal A} 
\big(N^{\mathcal T}_a(t\!-\!1) - W^{\mathcal T}_a(t\!-\!1)\big),
& \text{w.p. } 1-\epsilon_t, \\
a \sim \pi_0, 
& \text{w.p. } \epsilon_t,
\end{cases}\label{eq:action-choice}
\end{equation}
where $\epsilon_t = t^{-\alpha}$ for some $\alpha\in(0,1)$.
The forced-exploration component guarantees the non-degeneracy of the information matrix at only an asymptotically lower-order cost, while the tracking component ensures that the empirical proportions approach the optimal-design set.
Moreover, we assume the curvature is bounded away from zero on the relevant domain (\Cref{ass:curvature}), a standard condition in best arm identification that ensures each exploration action provides non-vanishing curvature~\citep{jun2021improved,kazerouni2021best}; in these works, the sample complexity also depends on $\kappa^{-1}$.

\begin{assumption}\label{ass:curvature}
There exists $\kappa>0$ such that
$\inf_{a\in\mathcal A,\bm\theta\in\Theta}
\mu'_{m(a)}(\bm x_a^\top\bm\theta)
\ge \kappa$.
\end{assumption}

\begin{remark}
With forced exploration, Assumption~\ref{ass:curvature} ensures information growth for consistency, almost-sure stopping, and asymptotic complexity. It is not needed for confidence coverage or error control at finite stopping times.
If minimax-optimal tracking alone ensures persistent excitation, i.e., the induced information matrix is almost surely full rank with $\lambda_{\min}(A_t) \gtrsim t^\alpha$ for some $\alpha>0$, then forced exploration is unnecessary and one may set $\epsilon_t \equiv 0$.
\end{remark}

\textbf{Stopping Rule.}
The algorithm stops at the first time $t$ such that the current empirically best arm $\hat i_t^\star$ is certified optimal against all competitors with high confidence:
\begin{equation}
\label{eq:stopping}
    \forall j\neq \hat i_t^\star:\quad
    \inf_{\bm{\theta} \in \mathcal E_t(\delta)}
    \bm{x}_{\hat i_t^\star, j}^\top \bm{\theta} > 0.
\end{equation}

\begin{theorem}[Correctness of the stopping rule]
The stopping rule~\eqref{eq:stopping} guarantees that, upon stopping, the returned arm satisfies $\hat i_\tau^\star = i^\star$ with probability at least $1-\delta$.
\end{theorem}

\begin{proof}
\Cref{eq:stopping} implies that if $\bm \theta^\star \in \mathcal E_t(\delta)$, \(
    \bm{x}_{\hat i_t^\star,j}^\top \bm{\theta}^\star > 0, \forall j \neq \hat i_t^\star,
\)
which yields $\hat i_t^\star = {\argmax}_i \bm{x}_i^\top \bm{\theta}^\star$. The time-uniform coverage of $\mathcal E_t(\delta)$ ensures this event holds with probability at least $1-\delta$.
\end{proof}

We now describe our proposed algorithm, Hybrid Track-and-Stop for GLBs, in~\Cref{alg:greedy}.
For warm-up, querying $d$ actions with linearly independent feature vectors once each is enough; \Cref{ass:curvature} then ensures $\bm A_{d+1}\succ0$, so $E=d$ suffices and is independent of $\delta$.
The constant $\kappa$ is an analytical property of the links and parameter domain, not an input or a tuning parameter of the algorithm.

\begin{algorithm}[t]
\caption{HyTS-GLB}
\label{alg:greedy}

\KwIn{Confidence level $\delta$, warm-up rounds $E$}

Warm-up: select actions for $E$ rounds to ensure $\bm A_{E+1}$ is positive definite\;

\For{$t = E+1, E+2, \ldots$}{
    Compute the MLE $\hat{\bm\theta}_t \gets {\argmin}_{\bm\theta \in \Theta} \mathcal L_t(\bm\theta)$\;
    
    \If{the stopping condition~\eqref{eq:stopping} holds}{
        \Return $\hat i_t^\star = {\argmax}_{i \in \mathcal{K}} \bm{x}_i^\top \hat{\bm\theta}_t$\;
    }

    Compute the target proportions $\bm w^\star(t)$ and then select action $a_t$ according to \Cref{eq:action-choice}\;
    
    Execute action $a_t$ and observe feedback $Z_t$, then record the information\;
}
\end{algorithm}

\begin{definition}[Minimax-Optimal characteristic time]
\label{def:characteristic-time}
For a given instance $\bm{\theta}^\star$, define the minimax characteristic time
\[
    T^\star(\bm{\theta}^\star)
    \triangleq
    \inf_{\bm w \in \Delta_{\mathcal A}}
    \frac{
        \max_{i \ne i^\star}
        \|\bm{x}_{i^\star, i}\|^2_{\bm A (\bm w, \bm \theta^\star)^{-1}}
    }{
        \Delta_{\min}^2
    },
\]
where
\(
\Delta_{\min}
\triangleq
\min_{i \neq i^\star}
\bm{x}_{i^\star, i}^\top \bm{\theta}^\star
\)
denotes the smallest optimality gap. 
\end{definition}
This design-dependent quantity measures the cost of statistically separating the optimal arm from suboptimal arms under the worst comparison induced by the confidence-width stopping rule.
Appendix~\ref{app:lower-bound} discusses an information-theoretic lower bound for hybrid GLM BAI and compares it directly with $T^\star(\bm{\theta}^\star)$ in the linear--Gaussian benchmark.
This supports interpreting HyTS-GLB as tracking the minimax allocation induced by our confidence-width criterion, while stopping-time guarantees still incur conservative constants and logarithmic factors.

\begin{theorem}[Almost sure sample complexity upper bound]\label{thm:stopping-time}
Assume~\Cref{ass:arm-bound,ass:diff-convex,ass:self-concordant,ass:theta-bound,ass:curvature} hold. Then HyTS-GLB's sample complexity $\tau$ verifies \[
    \Pr\!\left(\limsup_{\delta\to0}\frac{\tau}{\log(1/\delta)}\lesssim T^\star(\bm\theta^\star)\right) = 1.
\]
\end{theorem}

\begin{proof}[Proof sketch]
Ignoring the negligible forced exploration phase, we focus on tracking.
As the empirical proportions approach the design set, the largest comparison variance is asymptotically controlled by $\Delta_{\min}^2T^\star(\bm\theta^\star)/t$.
Stopping is ensured once the confidence radius is below $\Delta_{\min}/2$.
Hence $t \gtrsim T^\star(\bm\theta^\star)\beta_t(\delta)$ suffices, and the claim follows by a standard fixed-point argument.
\end{proof}

\Cref{thm:stopping-time} shows that HyTS-GLB achieves instance-dependent sample complexity governed by the minimax characteristic time $T^\star(\bm{\theta}^\star)$, which captures the joint information from reward and dueling feedback under the adopted confidence-width criterion.
Unlike single-modality strategies, the minimax design underlying $T^\star(\bm{\theta}^\star)$ balances absolute and relative information, enabling faster separation along informative directions.
Consequently, HyTS-GLB can strictly outperform reward-only or dueling-only methods when neither modality alone is sufficient.

\textbf{Relation to~\citet{jun2021improved}.}
They study transductive pure exploration under a common logistic observation model, using variance-dependent fixed-design confidence bounds and the phased elimination algorithm RAGE-GLM.
Taking their measurement set to contain both arm features and pairwise differences also covers the special case in which our two channels share the same logistic model.
Our framework allows the channels to have different GLM likelihoods and combines a time-uniform likelihood-ratio confidence sequence with Track-and-Stop allocation over the joint query space.

\subsection{Cost-Aware HyTS-GLB}
\label{sec:cost-aware-ts}

In many applications, different feedback modalities incur heterogeneous acquisition costs. 
To capture this asymmetry, we extend HyTS-GLB to a cost-aware setting, where the objective is to minimize the \emph{total acquisition cost} required to identify the best arm with high confidence. 
All statistical modeling, confidence sequences, and stopping rules remain unchanged from the time-minimization setting.

\textbf{Cost Model.}
Each action $a \in \mathcal A$ incurs a known acquisition cost $c_a > 0$.
For example, if all reward queries have cost $C_{\rm reward}$ and all dueling queries have cost $C_{\rm dueling}$, then $c_a = C_{\rm reward}$ for $a \in \mathcal K$ and $c_a = C_{\rm dueling}$ for $a \in \mathcal G$.
The cumulative cost up to time $t$ is \[
    B_t \triangleq \sum_{s=1}^t c_{a_s}.
\]

The goal is to design a $\delta$-correct strategy that minimizes $B_\tau$.

\textbf{Cost-normalized Design and Tracking.} Let $\bm c \in \mathbb R^{|\mathcal A|}$ denote the vector of action costs.
We introduce a cost-normalized design vector
$\bm p \in \mathbb R^{\mathcal A}_{\ge 0}$
satisfying $\langle \bm c, \bm p \rangle = 1$,
where $p_a$ represents the sampling intensity per unit cost.
The induced information matrix is $\bm A(\bm p,\bm\theta)$.
Given the current estimate $\hat{\bm\theta}_t$, the cost-aware minimax-optimal design is defined as
\begin{equation}
\bm p^\star(t)
=
\argmin_{\bm p \ge 0:\,\langle \bm c, \bm p \rangle = 1}
\max_{i \neq \hat i^\star_t}
\|\bm{x}_{\hat i^\star_t, i}\|^2_{\bm A(\bm p,\hat{\bm\theta}_t)^{-1}}.
\label{eq:cost-g-opt}
\end{equation}
To implement the design in a round-based procedure, we normalize
\(
\bar{\bm p}^\star(t) = \frac{\bm p^\star(t)}{\sum_{a\in\mathcal A} p^\star_a(t)}
\)
and apply the same tracking rule as in HyTS-GLB.

\begin{definition}[Cost-aware characteristic time]
\label{def:characteristic-time-cost}
The cost-aware characteristic time is
\[
T_{\mathrm{cost}}^\star(\bm{\theta}^\star)
\triangleq
\inf_{\bm p \ge 0:\,\langle \bm c, \bm p \rangle = 1}
\frac{
\max_{i \neq i^\star}
\|\bm{x}_{i^\star, i}\|^2_{\bm A(\bm p, \bm{\theta}^\star)^{-1}}
}{
\Delta_{\min}^2
}.
\]
\end{definition}

\begin{theorem}
\label{thm:stopping-time-cost}
Under the same assumptions as in \Cref{thm:stopping-time}. Then Cost-Aware HyTS-GLB's cost $B_\tau$ verifies
\[
    \Pr\!\left(\limsup_{\delta \to 0}\frac{B_\tau}{\log(1/\delta)}\lesssim T_{\rm cost}^\star(\bm\theta^\star)\right) = 1.
\]
\end{theorem}

Compared with standard HyTS-GLB, which implicitly assumes uniform query costs, the cost-aware variant optimizes sampling in a cost-normalized design space.
By minimizing the worst-case uncertainty per unit cost, the algorithm adaptively favors the more cost-efficient feedback modality while preserving the same confidence sequences and stopping rule.
When acquisition costs differ, this can reduce total cost, while recovering standard HyTS-GLB as a special case under uniform costs.

\section{Experiments}\label{sec:exp}

\definecolor{myblue}{RGB}{31,119,180}
\definecolor{myorange}{RGB}{255,127,14}
\definecolor{mygreen}{RGB}{44,160,44}
\definecolor{mypurple}{RGB}{148,103,189}
\definecolor{mybrown}{RGB}{140,86,75}

\usetikzlibrary{patterns}
\usepgfplotslibrary{fillbetween}

\begin{figure*}[t]\centering
\begin{subfigure}[t]{.49\linewidth}\centering
\begin{tikzpicture}
\begin{axis}[width=.94\linewidth,height=4.5cm,
xlabel={Dimension $d$ ($K=d+1$ arms)},
xlabel style={font=\footnotesize},
tick label style={font=\scriptsize},tick align=outside,ymode=log,
legend style={at={(0.5,1.02)},anchor=south,legend columns=3,font=\fontsize{7}{8}\selectfont,draw=none,/tikz/every even column/.append style={column sep=2pt}},
xtick={2,4,6,8,10,12},]
\addplot[name path=mainRAGEu,draw=none,forget plot] coordinates {(2,111428.016) (4,284951.745) (6,517735.695) (8,819409.057) (10,1199781.68) (12,1654059.05)};
\addplot[name path=mainRAGEl,draw=none,forget plot] coordinates {(2,98896.851) (4,261902.855) (6,489010.238) (8,795430.21) (10,1152866.39) (12,1598702.82)};
\addplot[myblue,fill opacity=.13,forget plot] fill between[of=mainRAGEu and mainRAGEl];
\addplot[color=myblue,thick,mark=triangle*,mark size=1.6pt] coordinates {(2,105162.433) (4,273427.3) (6,503372.967) (8,807419.633) (10,1176324.03) (12,1626380.93)};
\addlegendentry{RAGE-GLM}
\addplot[name path=mainReTSu,draw=none,forget plot] coordinates {(2,3768.97153) (4,70593.4795) (6,263799.002) (8,684344.309) (10,1543092.64) (12,2349762.83)};
\addplot[name path=mainReTSl,draw=none,forget plot] coordinates {(2,2069.69514) (4,25226.5205) (6,125071.664) (8,400506.358) (10,688494.026) (12,1537122.5)};
\addplot[mygreen,fill opacity=.13,forget plot] fill between[of=mainReTSu and mainReTSl];
\addplot[color=mygreen,thick,mark=square*,mark size=1.6pt] coordinates {(2,2919.33333) (4,47910) (6,194435.333) (8,542425.333) (10,1115793.33) (12,1943442.67)};
\addlegendentry{ReTS-GLB}
\addplot[name path=mainRandomu,draw=none,forget plot] coordinates {(2,3538.01528) (4,38688.0969) (6,125376.687) (8,306285.806) (10,585194.336) (12,1006491.59)};
\addplot[name path=mainRandoml,draw=none,forget plot] coordinates {(2,2315.31805) (4,18985.2364) (6,69763.3133) (8,174627.527) (10,283632.33) (12,606801.743)};
\addplot[mybrown,fill opacity=.13,forget plot] fill between[of=mainRandomu and mainRandoml];
\addplot[color=mybrown,thick,mark=o,mark size=1.6pt] coordinates {(2,2926.66667) (4,28836.6667) (6,97570) (8,240456.667) (10,434413.333) (12,806646.667)};
\addlegendentry{Random}
\addplot[name path=mainDuelTSu,draw=none,forget plot] coordinates {(2,3086.00371) (4,29485.7999) (6,116988.586) (8,282109.404) (10,466411.596) (12,1271058.24)};
\addplot[name path=mainDuelTSl,draw=none,forget plot] coordinates {(2,2672.66296) (4,20894.2001) (6,74748.7474) (8,110894.596) (10,209235.071) (12,233920.428)};
\addplot[mypurple,fill opacity=.13,forget plot] fill between[of=mainDuelTSu and mainDuelTSl];
\addplot[color=mypurple,thick,mark=diamond*,mark size=1.6pt] coordinates {(2,2879.33333) (4,25190) (6,95868.6667) (8,196502) (10,337823.333) (12,752489.333)};
\addlegendentry{DuelTS-GLB}
\addplot[name path=mainHyTSu,draw=none,forget plot] coordinates {(2,3412.57248) (4,25228.6612) (6,93433.6137) (8,215109.385) (10,378801.185) (12,728588.979)};
\addplot[name path=mainHyTSl,draw=none,forget plot] coordinates {(2,1659.42752) (4,15491.3388) (6,59497.0529) (8,137207.949) (10,187778.815) (12,268303.021)};
\addplot[myorange,fill opacity=.13,forget plot] fill between[of=mainHyTSu and mainHyTSl];
\addplot[color=myorange,thick,mark=*,mark size=1.6pt] coordinates {(2,2536) (4,20360) (6,76465.3333) (8,176158.667) (10,283290) (12,498446)};
\addlegendentry{HyTS-GLB}
\end{axis}
\end{tikzpicture}
\caption{Stopping queries versus dimension}
\end{subfigure}\hfill
\begin{subfigure}[t]{.49\linewidth}\centering
\begin{tikzpicture}
\begin{axis}[width=.94\linewidth,height=4.5cm,ybar,single ybar legend,bar width=7pt,
xlabel={Reward/dueling cost ratio},
xlabel style={font=\footnotesize},tick label style={font=\scriptsize},
xtick={1,2,3,4,5},xticklabels={$1/3$,$1/2$,$1$,$2$,$3$},ymin=0,enlarge x limits=.13,
legend style={at={(0.5,1.02)},anchor=south,legend columns=2,font=\fontsize{7}{8}\selectfont,draw=none},]
\addplot[fill=myorange!65,draw=myorange,bar shift=-4.5pt] coordinates {(1,2291.03333) (2,2372.68889) (3,2536) (4,2699.31111) (5,2780.96667)};
\addlegendentry{Standard dueling}
\addplot[fill=myblue!65,draw=myblue,bar shift=-4.5pt] coordinates {(1,756.483333) (2,1008.64444) (3,1512.96667) (4,2017.28889) (5,2269.45)};
\addlegendentry{Standard reward}
\addplot[black,ybar,fill=none,bar shift=-4.5pt,draw=none,forget plot,error bars/.cd,y dir=both,y explicit,error bar style={draw=black,thin},error mark options={rotate=90,mark size=2pt}] coordinates {(1,2291.03333) +- (0,456.132957) (2,2372.68889) +- (0,494.947359) (3,2536) +- (0,876.572479) (4,2699.31111) +- (0,1375.2161) (5,2780.96667) +- (0,1636.28449)};
\addplot[fill=myorange!65,draw=myorange,bar shift=4.5pt,postaction={pattern=north east lines}] coordinates {(1,1461.86667) (2,1943.77778) (3,2536) (4,1544.42222) (5,1280.4)};
\addlegendentry{Cost-aware dueling}
\addplot[fill=myblue!65,draw=myblue,bar shift=4.5pt,postaction={pattern=north east lines}] coordinates {(1,1446.06667) (2,1850.88889) (3,1512.96667) (4,205.288889) (5,68.6)};
\addlegendentry{Cost-aware reward}
\addplot[black,ybar,fill=none,bar shift=4.5pt,draw=none,forget plot,error bars/.cd,y dir=both,y explicit,error bar style={draw=black,thin},error mark options={rotate=90,mark size=2pt}] coordinates {(1,1461.86667) +- (0,422.632288) (2,1943.77778) +- (0,524.741289) (3,2536) +- (0,876.572479) (4,1544.42222) +- (0,351.576874) (5,1280.4) +- (0,261.627479)};
\end{axis}
\end{tikzpicture}
\caption{Acquisition cost by modality}
\end{subfigure}
\caption{Scalability and acquisition cost of hybrid-feedback learning. Results are averaged over 30 independent runs, with shaded regions and error bars indicating one standard deviation.}
\label{fig:exp-combined}
\end{figure*}

Our experiments examine two questions: whether hybrid feedback improves sample efficiency over single-modality baselines, and whether the cost-aware design adapts to heterogeneous acquisition costs. Due to the high cost of exact optimal allocation, we follow \citet{jedra2020optimal} and compute designs with Frank--Wolfe \citep{jaggi2013revisiting}. All methods are evaluated on the same generated instances under the same logistic observation model and confidence level. We set $K = d+1$ with $d\in\{2, 4, 6, 8, 10, 12\}$, $\delta = 0.05$, and $\bm{\theta}^\star = \frac{S-1}{\sqrt d} \cdot \mathbf{1}$ with $S = 5$. Arms are constructed so that $\bm x_1^\top \bm\theta^\star = 0.9\|\bm\theta^\star\|_2$, while for $i\ge2$, $\bm x_i^\top \bm\theta^\star = 0.8 \cdot \frac{K - i}{K - 1} \|\bm \theta^\star\|_2$, with remaining coordinates random.
We compare HyTS-GLB with RageGLM \citep{jun2021improved}, ReTS-GLB, DuelTS-GLB, and a random hybrid baseline. We also tested GLGapE \citep{kazerouni2021best}, but it did not terminate within the simulation budget, consistent with observations in \citet{jun2021improved}. Results report means over $30$ runs on shared instances. As shown in \Cref{fig:exp-combined}(a), HyTS-GLB consistently outperforms RageGLM, ReTS-GLB and DuelTS-GLB across all dimensions. This suggests that hybrid allocation is most useful when reward and dueling feedback reduce uncertainty along complementary directions, especially as the dimension grows.
We further evaluate cost-aware allocation with $C_{\text{reward}} + C_{\text{dueling}}=2$ and $\frac{C_{\text{reward}}}{C_{\text{dueling}}} \in \{\frac13, \frac12, 1, 2, 3\}$. \Cref{fig:exp-combined}(b) reports the reward and dueling cost decomposition for $K=3$, $d=2$. The cost-aware variant uniformly reduces total cost relative to standard HyTS-GLB by shifting samples toward the cheaper modality; when $\frac{C_{\text{reward}}}{C_{\text{dueling}}}=1$, the two coincide, matching our results in~\Cref{sec:cost-aware-ts}.

Beyond these main experiments, we provide additional evidence for the role of hybrid allocation in the appendix. We also give two concrete complexity comparisons: one where dueling feedback is more informative and one where reward feedback is more informative. These examples reinforce the main message that HyTS-GLB does not commit to a single feedback modality, but adapts its sampling allocation to the instance geometry and feedback utility. Code and raw results are available at GitHub\footnote{\href{https://github.com/qirunzeng/HyTS-GLB}{https://github.com/qirunzeng/HyTS-GLB}}.
\section{Concluding Remarks}
\label{sec:limitation}
We study fixed-confidence best-arm identification in generalized linear bandits with hybrid reward and dueling feedback. Our unified likelihood-based confidence sequence supports a geometry-aware Track-and-Stop algorithm with $\delta$-correctness and instance-dependent almost-sure asymptotic stopping-time bounds as $\delta\to0$. The framework allocates queries according to the information provided by each channel and extends to heterogeneous acquisition costs.

Several limitations remain. The two channels share a latent preference parameter; if they reflect different preferences, pooling their likelihoods may invalidate our guarantees. Computationally, repeated MLE and design updates over the $O(K^2)$ pairwise action space can make each query expensive. Pair screening, less frequent updates, and approximate design oracles may improve scalability, but their effects on stopping-time guarantees remain to be established. Our empirical evaluation is limited to synthetic logistic instances, and validation on real feedback data is still needed.

Future work includes closing the gap to fully instance-optimal guarantees for general GLMs, developing robust models with partially shared parameters, and extending the framework to time-varying costs, contextual or non-stationary environments, and richer preference signals.

\section*{Acknowledgement}
The work of Jinhang Zuo is supported by the Early Career Scheme (No. CityU 21212025) of the University Grants Committee (UGC) and CityU 9610706.
The work of Xuchuang Wang is supported in part by the RGC JRFS2526-4S06.

\bibliography{refs}
\bibliographystyle{plainnat}

\appendix
\clearpage
\section*{Appendix}
\phantomsection
\begingroup
\small
\tableofcontents
\endgroup

\clearpage
\section{Proofs}

\subsection{\texorpdfstring{Proof of~\Cref{lem:cs-hybrid}}{Proof of the hybrid likelihood-ratio confidence sequence}}
\label{app:proof-hybrid-cs}

Let $\mathcal F_t$ denote the history after observing $Z_t$.  At round
$t$, the learner selects $a_t$ using an $\mathcal F_{t-1}$-measurable
randomized policy, before observing $Z_t$.  Write
$m_t\triangleq m(a_t)$ and $\bm x_t\triangleq\bm x_{a_t}$.  For each
modality $m\in\{\cc,\dd\}$, let $\nu_m$ be a dominating measure for its
conditional distributions.  For any fixed $\bm\theta\in\Theta$, define
\[
    M_t(\bm\theta)
    \triangleq
    \exp\!\left\{
        \mathcal L_t(\bm\theta^\star)-\mathcal L_t(\bm\theta)
    \right\}
    =
    \prod_{s=1}^{t-1}
    \frac{
        p_{m_s}(Z_s\mid\bm x_s^\top\bm\theta)
    }{
        p_{m_s}(Z_s\mid\bm x_s^\top\bm\theta^\star)
    } .
\]
The base-measure terms in \Cref{eq:pdf} cancel in the likelihood
ratio.  Although the density and even the observation space may depend
on $m_s$, conditioning on the selected modality gives
\begin{align*}
    \mathbb E_{\bm\theta^\star}
    [M_{t+1}(\bm\theta)\mid\mathcal F_{t-1},a_t]
    &=
    M_t(\bm\theta)
    \int
    \frac{
        p_{m_t}(z\mid\bm x_t^\top\bm\theta)
    }{
        p_{m_t}(z\mid\bm x_t^\top\bm\theta^\star)
    }
    p_{m_t}(z\mid\bm x_t^\top\bm\theta^\star)
    \,\mathrm d\nu_{m_t}(z) = M_t(\bm\theta).
\end{align*}
Taking conditional expectation once more over the learner's action
randomization gives
$\mathbb E[M_{t+1}(\bm\theta)\mid\mathcal F_{t-1}]
=M_t(\bm\theta)$.  Hence
$\{M_t(\bm\theta)\}_{t\ge1}$ is a nonnegative martingale with respect to $\{\mathcal F_{t-1}\}_{t\ge1}$, with $M_1(\bm\theta)=1$.  This also
shows directly why adaptive switching between the two modalities does
not require a union bound.

We next apply the time-uniform PAC--Bayes argument.  As in~\citet{lee2024unified}, let $Q = \operatorname{Unif} (\Theta)$, where the uniform measure is taken on the affine hull of $\Theta$.  The volume scaling argument below then gives a KL term no larger than $d \log \frac{1}{c}$.  By Tonelli's theorem, $\overline M_t \triangleq \mathbb E_{\bm\theta\sim Q} [M_t(\bm \theta)]$ is also a nonnegative martingale with $\overline M_1 = 1$.  Ville's inequality therefore gives an event $\mathcal V_\delta$ with $\Pr (\mathcal V_\delta)\ge1-\delta$ on which, simultaneously for all $t \ge 1$,
\[
    \log\mathbb E_{\bm\theta\sim Q}[M_t(\bm \theta)] \le \log \tfrac1\delta .
\]
For any $\mathcal F_{t-1}$-measurable posterior $P_t\ll Q$, the Donsker--Varadhan variational inequality and the preceding display imply
\[
    \mathcal L_t(\bm\theta^\star)
    -
    \mathbb E_{\bm\theta\sim P_t}[\mathcal L_t(\bm\theta)]
    \le
    \log\tfrac1\delta+D_{\rm KL}(P_t\|Q),
    \quad\text{for all }t
\]
on $\mathcal V_\delta$.

Fix $c\in(0,1]$ and choose
\[
    P_{t,c}
    =
    \operatorname{Unif}\!\left(
        (1-c)\hat{\bm\theta}_t+c\Theta
    \right).
\]
Convexity of $\Theta$ ensures that this posterior is supported on $\Theta$, while scaling of Lebesgue volume yields
\[
    D_{\rm KL}(P_{t,c}\|Q)\le d\log\tfrac1c.
\]
Every $\bm\theta$ in its support can be written as $\bm \theta = (1-c) \hat{\bm \theta}_t + c \widetilde{\bm\theta}$ for some $\widetilde{\bm \theta} \in \Theta$, and consequently $\|\bm \theta - \hat{\bm\theta}_t\|_2\le2Sc$.  If $L_t$ is any pathwise upper bound on the Lipschitz modulus of $\mathcal L_t$ over $\Theta$, then
\[
    \mathbb E_{\bm\theta\sim P_{t,c}}
    [\mathcal L_t(\bm\theta)-\mathcal L_t(\hat{\bm\theta}_t)]
    \le 2SL_tc.
\]
Combining these inequalities, and then minimizing over $c \in (0,1]$, gives
\[
    \mathcal L_t(\bm\theta^\star)
    -
    \mathcal L_t(\hat{\bm\theta}_t)
    \le
    \log\tfrac1\delta
    +
    \inf_{c\in(0,1]}
    \left\{
        d\log\tfrac1c+2SL_tc
    \right\}
\]
simultaneously for all $t$ on $\mathcal V_\delta$.  Choosing $c = \min\{1, \frac{d}{2 S L_t}\}$ gives the final upper bound in \Cref{eq:beta-hybrid} and proves~\Cref{lem:cs-hybrid}.

Following the Lipschitz-bound argument of \citet{lee2024unified}, convexity of $\Theta$ and the mean-value and triangle inequalities give, on every sample path,
\[
\sup_{\substack{\bm\theta,\bm\theta'\in\Theta\\
          \bm\theta\ne\bm\theta'}}
\frac{
\left|\mathcal L_t(\bm\theta)-\mathcal L_t(\bm\theta')\right|
}{
\|\bm\theta-\bm\theta'\|_2
}
\le
\sum_{s=1}^{t-1}
\frac{\rho_{a_s}}{\zeta_{m(a_s)}(\varphi_{m(a_s)})}
\left(
|Z_s|
+
\sup_{\bm\vartheta\in\Theta}
\left|\mu_{m(a_s)}(\bm x_{a_s}^{\top}\bm\vartheta)\right|
\right).
\]
The inner supremum is finite by compactness and continuity, so the right-hand side is an admissible choice of $L_t$ almost surely. Consistent with the linear-growth bounds in~\citet[Table~1 and Appendix~C]{lee2024unified}, this choice also satisfies $L_t=O(t)$ almost surely: under our assumptions, the observations have uniformly bounded conditional second moments, and the martingale strong law gives $\sum_{s=1}^{t-1} |Z_s|=O(t)$ almost surely. The feature-norm bounds $1$ for reward feedback and $2$ for dueling feedback affect only constants.

\subsection{\texorpdfstring{Proof of~\Cref{thm:ellip-hybrid}}{Proof of the ellipsoidal confidence bound}}
We follow the proof of~\citet[Theorem~3.2]{lee2024unified}.

Define
\begin{align*}
    &\tilde G^\cc_t (\bm{\theta}, \bm\nu) \triangleq \frac{1}{\zeta_c (\varphi_c)} \sum_{s \in \mathcal C_t} \tilde \alpha^\cc_s (\bm{\theta}, \bm \nu) \bm x_{i_s} \bm x_{i_s}^\top, \\
    &\tilde G^\dd_t (\bm{\theta}, \bm \nu) \triangleq \frac{1}{\zeta_d (\varphi_d)} \sum_{s \in \mathcal D_t} \tilde \alpha^\dd_s (\bm{\theta}, \bm\nu) \bm x_{a_s} \bm x_{a_s}^\top, 
\end{align*}
where
\begin{align*}
    \tilde \alpha^\cc_s (\bm{\theta},\bm \nu) &\triangleq \int_0^1 (1-v) \mu_\cc' (\langle \bm x_{a_s}, \bm{\theta} + v(\bm \nu - \bm{\theta}) \rangle) \mathrm{d} v, \\
    \tilde \alpha^\dd_s (\bm{\theta}, \bm \nu) &\triangleq \int_0^1 (1-v) \mu_\dd' (\langle \bm x_{a_s}, \bm{\theta} + v(\bm \nu - \bm{\theta}) \rangle) \mathrm{d} v.
\end{align*}

\begin{lemma}[{\citealp[Lemma~D.1]{lee2024unified}}]
    Let $\mu_\cc, \mu_\dd$ be increasing ($\mu_\cc', \mu_\dd' \ge 0$) and self-concordant with constants $M_{\cc}, M_{\dd}$. Let $\mathcal Z^\cc, \mathcal Z^\dd \subset \mathbb R$ be bounded. Then, for any $z_1^\cc, z_2^\cc \in \mathcal Z^\cc$ and $z_1^\dd, z_2^\dd \in \mathcal Z^\dd$,
    \begin{align*}
        \int_0^1 (1-v) \mu_\cc' (z_1^\cc + v(z_2^\cc - z_1^\cc)) \mathrm{d} v &\ge \frac{\mu_\cc'(z_1^\cc)}{2 + M_{\cc} |z_1^\cc - z_2^\cc|}, \\
        \int_0^1 (1-v) \mu_\dd' (z_1^\dd + v(z_2^\dd - z_1^\dd)) \mathrm{d} v &\ge \frac{\mu_\dd'(z_1^\dd)}{2 + M_{\dd} |z_1^\dd - z_2^\dd|}.
    \end{align*}
    Consequently, since $|z_1-z_2| \le \|\bm x_i^\top\|_2 \|\bm\nu-\bm{\theta}\|_2 \le 2S$ for reward observations and $|z_1-z_2| \le \|\bm x_{j,k}^\top\|_2 \|\bm\nu-\bm{\theta}\|_2 \le 4S$ for dueling observations,
    \begin{align*}
        \tilde G_t^\cc (\bm{\theta}, \bm\nu) &\succeq \frac{\nabla^2\mathcal L^\cc_t(\bm{\theta})}{2(1 + SM_{\cc})},\\
        \tilde G_t^\dd (\bm{\theta}, \bm\nu) &\succeq \frac{\nabla^2\mathcal L^\dd_t(\bm{\theta})}{2(1 + 2SM_{\dd})}.
    \end{align*}
\end{lemma}

Let \[
    \mathcal L_t^\cc \triangleq \sum_{s \in \mathcal C_t}
    { (b_{\cc}(\bm{x}_{a_s}^\top \bm\theta) - R_s \bm x_{a_s}^\top \bm \theta) }/{\zeta_{\cc}(\varphi_{\cc})},
\]
\[
    \mathcal L_t^\dd \triangleq \sum_{s \in \mathcal D_t} { (b_{\dd}(\bm x_{a_s}^\top \bm\theta) - Y_s \bm x_{a_s}^\top \bm \theta) }/{\zeta_{\dd}(\varphi_{\dd})}
\]
denote the cumulative loss incurred by reward observation and dueling observation separately.
By Taylor's theorem with integral remainder,
\begin{align*}
    \mathcal L^\cc_t (\bm{\theta}) - \mathcal L^\cc_t(\hat{\bm{\theta}}_t) 
    &= \langle \nabla \mathcal L^\cc_t (\hat{\bm{\theta}}_t), \bm{\theta} - \hat{\bm{\theta}}_t\rangle 
    + \|\bm{\theta} - \hat{\bm{\theta}}_t\|^2_{\tilde G^\cc_t (\hat{\bm{\theta}}_t, \bm{\theta})},\\
    \mathcal L^\dd_t (\bm{\theta}) - \mathcal L^\dd_t(\hat{\bm{\theta}}_t) 
    &= \langle \nabla \mathcal L^\dd_t (\hat{\bm{\theta}}_t), \bm{\theta} - \hat{\bm{\theta}}_t\rangle 
    + \|\bm{\theta} - \hat{\bm{\theta}}_t\|^2_{\tilde G^\dd_t (\hat{\bm{\theta}}_t, \bm{\theta})}.
\end{align*}
Summing the two equalities yields
\begin{align*}
    \mathcal L_t(\bm{\theta}) - \mathcal L_t(\hat{\bm{\theta}}_t)
    = \langle \nabla \mathcal L_t(\hat{\bm{\theta}}_t), \bm{\theta} - \hat{\bm{\theta}}_t\rangle
    + \|\bm{\theta} - \hat{\bm{\theta}}_t\|^2_{\tilde G^\cc_t (\hat{\bm{\theta}}_t, \bm{\theta})}
    + \|\bm{\theta} - \hat{\bm{\theta}}_t\|^2_{\tilde G^\dd_t (\hat{\bm{\theta}}_t, \bm{\theta})}.
\end{align*}
Since $\hat{\bm{\theta}}_t \in {\argmin}_{\bm{\theta}\in\Theta}\mathcal L_t(\bm{\theta})$ and $\Theta$ is convex, the first-order optimality condition implies that for any $\bm{\theta}\in\Theta$,
\[
\langle \nabla \mathcal L_t(\hat{\bm{\theta}}_t), \bm{\theta} - \hat{\bm{\theta}}_t\rangle \ge 0.
\]
Therefore,
\begin{align*}
    \mathcal L_t(\bm{\theta}) - \mathcal L_t(\hat{\bm{\theta}}_t)
    &\ge \|\bm{\theta} - \hat{\bm{\theta}}_t\|^2_{\tilde G^\cc_t (\hat{\bm{\theta}}_t, \bm{\theta})}
    + \|\bm{\theta} - \hat{\bm{\theta}}_t\|^2_{\tilde G^\dd_t (\hat{\bm{\theta}}_t, \bm{\theta})}.
\end{align*}

\begin{lemma}
    \label{lem:split-sc-quad}
    For $t \ge 1$, consider a hybrid cumulative loss \[
        \mathcal L_t(\bm{\theta}) = \mathcal L_t^\cc(\bm{\theta}) + \mathcal L_t^\dd(\bm{\theta}),
    \]
    with $\bm H_t^\cc(\bm{\theta}) \triangleq \nabla^2 \mathcal L_t^\cc (\bm{\theta})$ and $\bm H_t^\dd (\bm{\theta}) \triangleq \nabla^2 \mathcal L_t^\dd (\bm{\theta})$.
    Let $\hat{\bm{\theta}}_t \in \arg\min_{\bm{\theta} \in \Theta} \mathcal L_t (\bm{\theta})$. Then for all $\bm{\theta} \in \Theta$, \[
        \mathcal L_t (\bm{\theta}) - \mathcal L_t(\hat{\bm{\theta}}_t) \ge \|\bm{\theta} - \hat{\bm{\theta}}_t\|^2_{\bm A_t}.
    \]
\end{lemma}

\begin{proof}[Proof of~\Cref{lem:split-sc-quad}]
    Combining the above inequalities yields
    \begin{align*}
        \mathcal L_t(\bm{\theta})-\mathcal L_t(\hat{\bm{\theta}}_t)
        &\ge
        \frac{\|\bm{\theta}-\hat{\bm{\theta}}_t\|^2_{\bm H_t^\cc(\hat{\bm{\theta}}_t)}}{2(1+SM_{\cc})}
        +
        \frac{\|\bm{\theta}-\hat{\bm{\theta}}_t\|^2_{\bm H_t^\dd(\hat{\bm{\theta}}_t)}}{2(1+2SM_{\dd})} \\
        &= (\bm{\theta}-\hat{\bm{\theta}}_t)^\top \left(\frac{\bm H_t^\cc (\hat{\bm{\theta}}_t)}{2(1 + S M_{\cc})} + \frac{\bm H_t^\dd(\hat{\bm{\theta}}_t)}{2(1 + 2S M_{\dd})} \right) (\bm{\theta}-\hat{\bm{\theta}}_t) \\
        &= \|\bm{\theta} - \hat{\bm{\theta}}_t\|^2_{\bm A_t},
    \end{align*}
    as claimed.
\end{proof}

\begin{proof}[Proof of~\Cref{thm:ellip-hybrid}]
    By~\Cref{lem:split-sc-quad},
    \[
        \|\bm{\theta}^\star - \hat{\bm{\theta}}_t\|^2_{\bm A_t} \le \mathcal L_t (\bm{\theta}^\star) - \mathcal L_t(\hat{\bm{\theta}}_t) \le \beta_t(\delta).
    \]
\end{proof}

\subsection{\texorpdfstring{Proof of~\Cref{thm:stopping-time}}{Proof of the stopping-time bound}}

{Write $n_t=|\mathcal T_t|$ for the number of tracking rounds and consider a common infinite sampling trajectory, with the sampling rule and tie-breaking independent of $\delta$.
For a spanning exploration distribution $\pi_0$ (e.g., uniform on $\mathcal A_0$), the strong law for exploration counts and \Cref{ass:curvature} give $\frac{n_t}{t}\to1$ and $\lambda_{\min}(\bm A_t)\gtrsim t^{1-\alpha}$ a.s.
Applying~\Cref{thm:ellip-hybrid} with $\delta_t=t^{-2}$ and using Borel--Cantelli and $L_t=O(t)$ a.s.\ yields $\|\hat{\bm\theta}_t-\bm\theta^\star\|_2^2=O(\frac{\log t}{t^{1-\alpha}})\to0$ a.s.}

\begin{lemma}[Adapted from Lemmas~1 and~2 of~\citet{jedra2020optimal}]\label{lem:w-asym}
    {Under the assumptions of~\Cref{thm:stopping-time}, $\phi$ is jointly continuous wherever the best arm is unchanged and $\bm A(\bm w,\bm\theta)\succ0$.
    Define
    \[
        \phi^\star(\bm\theta)
        \triangleq
        \min_{\bm w\in\Delta_{\mathcal A}}\phi(\bm\theta,\bm w),
        \qquad
        \mathcal W^\star(\bm\theta)
        \triangleq
        \argmin_{\bm w\in\Delta_{\mathcal A}}\phi(\bm\theta,\bm w).
    \]
    Then $\phi^\star$ is continuous at $\bm\theta^\star$, and $\mathcal W^\star(\bm\theta^\star)$ is nonempty, compact, and convex, with every optimal design positive definite.
    The optimal-design correspondence is upper hemicontinuous at $\bm\theta^\star$.
    The same properties hold on the cost-normalized domain $\{\bm p\ge0:\langle\bm c,\bm p\rangle=1\}$.}
\end{lemma}

\begin{proof}
{Following the compactness and maximum-theorem argument of~\citet[Lemmas~1 and~2]{jedra2020optimal}, it suffices to exclude singular limits of near-optimal designs.
The dueling spanning assumption implies $\bm G\triangleq\sum_{i\ne i^\star}\bm x_{i^\star,i}\bm x_{i^\star,i}^\top\succ0$, and, for $\bm A\succ0$,
\[
    \max_{i\ne i^\star}\bm x_{i^\star,i}^\top\bm A^{-1}\bm x_{i^\star,i}
    \ge\frac{\operatorname{tr}(\bm A^{-1}\bm G)}{K-1}
    \ge\frac{\lambda_{\min}(\bm G)}{(K-1)\lambda_{\min}(\bm A)}.
\]
Since the best arm is locally unchanged and a fixed spanning design has a finite continuous objective, this bound confines minimizing sequences near $\bm\theta^\star$ to a common compact subset of positive-definite designs.
Continuity, convexity in the design weights, and the maximum theorem then give the claims, on either compact design domain.}
\end{proof}

{Under the design regularity in~\Cref{lem:w-asym}, consistency makes the target designs approach $\mathcal W^\star(\bm\theta^\star)$.
The deterministic set-tracking argument of~\citet[Lemma~6 and Proposition~1]{jedra2020optimal}, applied to the tracking subsequence, therefore gives
\[
    \operatorname{dist}\!\left(
        \frac{\bm N^{\mathcal T}(t)}{n_t},
        \mathcal W^\star(\bm\theta^\star)
    \right)\longrightarrow0
    \qquad\text{a.s.}
\]
By continuity of the information weights and inverse matrices near this compact set of positive-definite optimal designs, $\bm A_t \succeq \bm A_t^{\mathcal T}$ and $\frac{n_t}{t} \to 1$ imply
\begin{equation}\label{eq:asym-width}
    \limsup_{t\to\infty} t\max_{i\ne i^\star}
    \|\bm x_{i^\star,i}\|_{\bm A_t^{-1}}^2
    \le \Delta_{\min}^2 T^\star(\bm\theta^\star)
    \qquad\text{a.s.}
\end{equation}}

\begin{proof}[Proof of~\Cref{thm:stopping-time}]
{Fix a trajectory on which the preceding almost-sure conclusions hold.
Consistency and the unique best arm imply that, after a finite random time independent of $\delta$, $\hat i_t^\star=i^\star$ and $\bm x_{i^\star,i}^\top\hat{\bm\theta}_t>\frac{\Delta_i}{2}$ for every $i\ne i^\star$.
Thus the stopping rule holds whenever
\[
    \beta_t(\delta)\max_{i\ne i^\star}
    \|\bm x_{i^\star,i}\|_{\bm A_t^{-1}}^2
    <\frac{\Delta_{\min}^2}{4}.
\]
For any fixed $C>4$, set
$t_\delta = \lceil C T^\star (\bm\theta^\star) \log(1/\delta)\rceil$.
\Cref{eq:beta-hybrid} and $L_t=O(t)$ imply $\frac{\beta_{t_\delta}(\delta)}{\log(1/\delta)}\to1$; hence \Cref{eq:asym-width} makes the preceding stopping condition hold at $t_\delta$ for all sufficiently small $\delta$.
Consequently, $\tau\le t_\delta$ eventually and
\[
    \limsup_{\delta\to0}\frac{\tau}{\log(1/\delta)}
    \le 4T^\star(\bm\theta^\star)
    \qquad\text{a.s.},
\]
which proves the claimed bound.}
\end{proof}

\subsection{\texorpdfstring{Proof of~\Cref{thm:stopping-time-cost}}{Proof of the cost bound}}

\begin{proof}
{Apply the same set-tracking argument to the normalized cost-optimal designs, using~\Cref{lem:w-asym} on $\{\bm p\ge0:\langle\bm c,\bm p\rangle=1\}$.
Writing $P_s=\sum_a p_a^\star(s-1)$, the cost-normalized cumulative target is
\[
    \bar{\bm p}_t
    =\frac{\sum_{s\in\mathcal T_t}\bm p^\star(s-1)/P_s}
    {\sum_{s\in\mathcal T_t}1/P_s},
\]
a convex combination of the target designs, which therefore approaches the convex cost-optimal set.
Since $\frac{1}{P_s}\in[c_{\min},c_{\max}]$, bounded tracking errors and negligible forced exploration give $\frac{B_t}{\sum_{s\in\mathcal T_t}(1/P_s)}\to1$ a.s., so the empirical cost-normalized design has the same limiting optimal set.
Thus the analogue of \Cref{eq:asym-width} is
\[
    \limsup_{t\to\infty}
    B_t\max_{i\ne i^\star}
    \|\bm x_{i^\star,i}\|_{\bm A_t^{-1}}^2
    \le \Delta_{\min}^2 T_{\mathrm{cost}}^\star(\bm\theta^\star)
    \qquad\text{a.s.}
\]
For $C>4$, let $s_\delta$ be the first time the cumulative cost exceeds
$b_\delta=C T_{\mathrm{cost}}^\star(\bm\theta^\star)\log(1/\delta)$.
The bounded overshoot gives $b_\delta<B_{s_\delta}\le b_\delta+c_{\max}$ and $s_\delta=O(\log(1/\delta))$, so $\frac{\beta_{s_\delta}(\delta)}{\log(1/\delta)}\to1$ a.s.
The same pathwise stopping condition therefore yields $\tau\le s_\delta$ for all sufficiently small $\delta$, and hence
\[
    \limsup_{\delta\to0}\frac{B_\tau}{\log(1/\delta)}
    \le4T_{\mathrm{cost}}^\star(\bm\theta^\star)
    \qquad\text{a.s.},
\]
as claimed.}
\end{proof}

\section{KL Lower Bound and Gap-Ratio Comparison}
\label{app:lower-bound}

Let $\mathcal M\subseteq\Theta$ be the set of parameters with a unique best arm. Define
\[
    \mathrm{Alt}(\bm\theta^\star)
    =\{\bm\lambda\in\mathcal M:i^\star(\bm\lambda)\ne i^\star(\bm\theta^\star)\}.
\]
Write $d_a(\bm\theta,\bm\lambda)$ for the KL divergence from $p_{m(a)}(\cdot\mid\bm x_a^\top\bm\theta)$ to $p_{m(a)}(\cdot\mid\bm x_a^\top\bm\lambda)$.

\begin{theorem}[Information-theoretic lower bound for hybrid GLM BAI]
\label{thm:hybrid-glm-lower}
For $\bm\theta^\star\in\mathcal M$, define
\[
    T_{\mathrm{LB}}(\bm\theta^\star)
    =\left[
        \sup_{\bm w\in\Delta_{\mathcal A}}
        \inf_{\bm\lambda\in\mathrm{Alt}(\bm\theta^\star)}
        \sum_{a\in\mathcal A}w_a d_a(\bm\theta^\star,\bm\lambda)
    \right]^{-1}.
\]
For $\delta\in(0,1/2)$, any $\delta$-correct strategy on $\mathcal M$ satisfies
\[
    \mathbb E_{\bm\theta^\star}[\tau_\delta]
    \ge T_{\mathrm{LB}}(\bm\theta^\star)\,\mathrm{kl}(1-\delta,\delta),
\]
where $\mathrm{kl}$ is binary relative entropy, with $\inf\varnothing=+\infty$, $1/(+\infty)=0$, and $1/0=+\infty$.
\end{theorem}

\begin{proof}
This follows directly from the standard change-of-measure lower bound~\citep{kaufmann2016complexity,degenne2019pure}, applied to the joint action space $\mathcal A$.
\end{proof}

\textbf{Linear--Gaussian comparison.}
Consider the homoscedastic linear--Gaussian benchmark with distinct arm features and parameter space $\mathbb R^d$: $Z\mid a\sim\mathcal N(\bm x_a^\top\bm\theta,\sigma^2)$, including real-valued utility-difference feedback for pair queries.
Here the existing design matrix reduces to $\bm A(\bm w)=(2\sigma^2)^{-1}\sum_{a\in\mathcal A}w_a\bm x_a\bm x_a^\top$.
Writing $\Delta_i=\bm x_{i^\star,i}^\top\bm\theta^\star$, the exact quadratic KL divergence gives~\citep{jedra2020optimal}
\[
    T_{\mathrm{LB}}(\bm\theta^\star)
    =\inf_{\bm w\in\Delta_{\mathcal A}}
    \max_{i\ne i^\star}
    \frac{\|\bm x_{i^\star,i}\|_{\bm A(\bm w)^{-1}}^2}{\Delta_i^2}.
\]
Compared with $T^\star$ in \Cref{def:characteristic-time}, the only change is the use of individual gaps instead of the common $\Delta_{\min}$.
Thus, with $\Delta_{\max}=\max_{i\ne i^\star}\Delta_i$, pointwise comparison of the design objectives yields
\[
    1\le
    \frac{T^\star(\bm\theta^\star)}{T_{\mathrm{LB}}(\bm\theta^\star)}
    \le\left(\frac{\Delta_{\max}}{\Delta_{\min}}\right)^2.
\]
Equal suboptimal gaps therefore make the two characteristic times coincide in this benchmark.

\section{Additional Experiments}
\label{app:additional-experiments}

\begin{figure*}[!htbp]\centering
\begin{subfigure}[t]{.49\linewidth}\centering
\begin{tikzpicture}
\begin{axis}[width=.94\linewidth,height=4.2cm,
xlabel={Dimension $d$ ($K=d+1$ arms)},
xlabel style={font=\footnotesize},
tick label style={font=\scriptsize},tick align=outside,ymode=log,
legend image post style={xscale=.9},
legend style={at={(0.5,1.02)},anchor=south,legend columns=3,font=\fontsize{7}{8}\selectfont,draw=none,/tikz/every even column/.append style={column sep=0.5pt}},
xtick={2,3,4,5,6,7,8},]
\addplot[name path=basisRAGEu,draw=none,forget plot] coordinates {(2,266288.415) (3,366419.587) (4,486237.264) (5,623266.069) (6,780328.186) (7,953463.51) (8,1153226.89)};
\addplot[name path=basisRAGEl,draw=none,forget plot] coordinates {(2,261975.985) (3,360976.68) (4,479549.803) (5,616951.064) (6,772698.281) (7,946065.957) (8,1143236.45)};
\addplot[myblue,fill opacity=.13,forget plot] fill between[of=basisRAGEu and basisRAGEl];
\addplot[color=myblue,thick,mark=triangle*,mark size=1.6pt] coordinates {(2,264132.2) (3,363698.133) (4,482893.533) (5,620108.567) (6,776513.233) (7,949764.733) (8,1148231.67)};
\addlegendentry{RAGE-GLM}
\addplot[name path=basisReTSu,draw=none,forget plot] coordinates {(2,220473.795) (3,419855.338) (4,449346.71) (5,1004720.27) (6,1079782.76) (7,1097978.25) (8,1126269.17)};
\addplot[name path=basisReTSl,draw=none,forget plot] coordinates {(2,174004.871) (3,351701.995) (4,397293.29) (5,857406.401) (6,929234.577) (7,984600.416) (8,1024074.83)};
\addplot[mygreen,fill opacity=.13,forget plot] fill between[of=basisReTSu and basisReTSl];
\addplot[color=mygreen,thick,mark=square*,mark size=1.6pt] coordinates {(2,197239.333) (3,385778.667) (4,423320) (5,931063.333) (6,1004508.67) (7,1041289.33) (8,1075172)};
\addlegendentry{ReTS-GLB}
\addplot[name path=basisRandomu,draw=none,forget plot] coordinates {(2,264154.238) (3,487747.577) (4,783670.743) (5,1114096.95) (6,1515470.54) (7,1990310.87) (8,2521036.31)};
\addplot[name path=basisRandoml,draw=none,forget plot] coordinates {(2,212579.095) (3,425585.756) (4,669729.257) (5,983969.713) (6,1372596.13) (7,1788155.8) (8,2337763.69)};
\addplot[mybrown,fill opacity=.13,forget plot] fill between[of=basisRandomu and basisRandoml];
\addplot[color=mybrown,thick,mark=o,mark size=1.6pt] coordinates {(2,238366.667) (3,456666.667) (4,726700) (5,1049033.33) (6,1444033.33) (7,1889233.33) (8,2429400)};
\addlegendentry{Random}
\addplot[name path=basisHyTSu,draw=none,forget plot] coordinates {(2,218032.62) (3,423391.128) (4,442936.509) (5,1007379.89) (6,1064501.68) (7,1125920.99) (8,1141844.06)};
\addplot[name path=basisHyTSl,draw=none,forget plot] coordinates {(2,171579.38) (3,354832.872) (4,402236.824) (5,870946.775) (6,937915.652) (7,1002657.68) (8,1032033.27)};
\addplot[myorange,fill opacity=.13,forget plot] fill between[of=basisHyTSu and basisHyTSl];
\addplot[color=myorange,thick,mark=*,mark size=1.6pt] coordinates {(2,194806) (3,389112) (4,422586.667) (5,939163.333) (6,1001208.67) (7,1064289.33) (8,1086938.67)};
\addlegendentry{HyTS-GLB}
\end{axis}
\end{tikzpicture}
\caption{Basis-plus-rotated arms, $S=5$}
\end{subfigure}\hfill
\begin{subfigure}[t]{.49\linewidth}\centering
\begin{tikzpicture}
\begin{axis}[width=.94\linewidth,height=4.2cm,
xlabel={Parameter bound $S$},
xlabel style={font=\footnotesize},
tick label style={font=\scriptsize},tick align=outside,ymode=log,
legend image post style={xscale=.9},
legend style={at={(0.5,1.02)},anchor=south,legend columns=3,font=\fontsize{7}{8}\selectfont,draw=none,/tikz/every even column/.append style={column sep=0.5pt}},
xtick={2.0,3.0,4.0,5.0},]
\addplot[name path=scaleRAGEu,draw=none,forget plot] coordinates {(2.0,119570.277) (3.0,63370.5864) (4.0,125069.691) (5.0,284951.745)};
\addplot[name path=scaleRAGEl,draw=none,forget plot] coordinates {(2.0,92224.5897) (3.0,58621.6136) (4.0,115868.709) (5.0,261902.855)};
\addplot[myblue,fill opacity=.13,forget plot] fill between[of=scaleRAGEu and scaleRAGEl];
\addplot[color=myblue,thick,mark=triangle*,mark size=1.6pt] coordinates {(2.0,105897.433) (3.0,60996.1) (4.0,120469.2) (5.0,273427.3)};
\addlegendentry{RAGE-GLM}
\addplot[name path=scaleReTSu,draw=none,forget plot] coordinates {(2.0,150781.295) (3.0,68254.1586) (4.0,56395.4744) (5.0,70593.4795)};
\addplot[name path=scaleReTSl,draw=none,forget plot] coordinates {(2.0,85998.7054) (3.0,35732.5081) (4.0,28391.1923) (5.0,25226.5205)};
\addplot[mygreen,fill opacity=.13,forget plot] fill between[of=scaleReTSu and scaleReTSl];
\addplot[color=mygreen,thick,mark=square*,mark size=1.6pt] coordinates {(2.0,118390) (3.0,51993.3333) (4.0,42393.3333) (5.0,47910)};
\addlegendentry{ReTS-GLB}
\addplot[name path=scaleRandomu,draw=none,forget plot] coordinates {(2.0,128321.11) (3.0,53985.3469) (4.0,40743.1514) (5.0,38688.0969)};
\addplot[name path=scaleRandoml,draw=none,forget plot] coordinates {(2.0,68632.223) (3.0,27581.3198) (4.0,20596.8486) (5.0,18985.2364)};
\addplot[mybrown,fill opacity=.13,forget plot] fill between[of=scaleRandomu and scaleRandoml];
\addplot[color=mybrown,thick,mark=o,mark size=1.6pt] coordinates {(2.0,98476.6667) (3.0,40783.3333) (4.0,30670) (5.0,28836.6667)};
\addlegendentry{Random}
\addplot[name path=scaleDuelTSu,draw=none,forget plot] coordinates {(2.0,108978.071) (3.0,45174.0697) (4.0,37785.6549) (5.0,29485.7999)};
\addplot[name path=scaleDuelTSl,draw=none,forget plot] coordinates {(2.0,78695.2623) (3.0,39365.9303) (4.0,25754.3451) (5.0,20894.2001)};
\addplot[mypurple,fill opacity=.13,forget plot] fill between[of=scaleDuelTSu and scaleDuelTSl];
\addplot[color=mypurple,thick,mark=diamond*,mark size=1.6pt] coordinates {(2.0,93836.6667) (3.0,42270) (4.0,31770) (5.0,25190)};
\addlegendentry{DuelTS-GLB}
\addplot[name path=scaleHyTSu,draw=none,forget plot] coordinates {(2.0,87948.2619) (3.0,37226.6616) (4.0,29090.2565) (5.0,25228.6612)};
\addplot[name path=scaleHyTSl,draw=none,forget plot] coordinates {(2.0,56618.4047) (3.0,23813.3384) (4.0,20023.0768) (5.0,15491.3388)};
\addplot[myorange,fill opacity=.13,forget plot] fill between[of=scaleHyTSu and scaleHyTSl];
\addplot[color=myorange,thick,mark=*,mark size=1.6pt] coordinates {(2.0,72283.3333) (3.0,30520) (4.0,24556.6667) (5.0,20360)};
\addlegendentry{HyTS-GLB}
\end{axis}
\end{tikzpicture}
\caption{Random arms, $(K,d)=(5,4)$}
\end{subfigure}
\caption{Mean stopping queries over 30 runs, with shading showing one standard deviation.}
\label{fig:appendix-additional-experiments}
\end{figure*}

We present two additional synthetic experiments that complement the random-arm results in~\Cref{sec:exp}.
All experiments use the same logistic observation model, confidence level, implementation of the stopping rule, and Frank--Wolfe design solver as in the main text.

In \Cref{fig:appendix-additional-experiments}(a), the arm set is
\[
    \mathcal{X} = \{ \bm e_1, \bm e_2, \ldots, \bm e_d, \cos(0.1)\bm e_1 + \sin(0.1)\bm e_2 \},
\]
with $\bm \theta = (S-1)\bm e_1$, $d \in \{2,3,4,5,6,7,8\}$ and $S = 5$.
This instance creates a hard comparison between the best basis arm and a nearby rotated arm.
The relevant separating direction has a small component along $\bm e_1-\cos(0.1)\bm e_1$ and a larger component along $\bm e_2$, so reward queries on basis arms can be highly informative even when the final decision is pairwise.
We do not report the dueling-only variant because it required more than $20,\!000,\!000$ samples in this construction, exceeding our computational limit. HyTS-GLB has mean stopping times close to those of ReTS-GLB and below the random hybrid baseline across the displayed dimensions.
Its mean stopping time is lower than RAGE-GLM at $d=2,4,8$.

In \Cref{fig:appendix-additional-experiments}(b), we return to the random-arm construction from the main text, fix $(K,d)=(5,4)$, and vary $S \in \{2,3,4,5\}$.
Changing $S$ changes the magnitude of the linear scores and therefore the local curvature of the logistic model.
HyTS-GLB achieves the smallest mean stopping time for all tested values of $S$, with the largest relative gains at larger $S$.

\section{Concrete Complexity Comparisons}
\label{app:symbolic-complexity}

In this section, we present two concrete examples showing that either reward or dueling feedback can be more informative, depending on the arm geometry. These examples motivate adaptively allocating queries between the two feedback channels.
Both channels are logistic with unit dispersion. Let $\Theta=\mathcal B^2(S)$ and $\bm\theta^\star=S\bm e_1$, with $S\ge4$, $B_\cc=2(1+S)$, and $B_\dd=2(1+2S)$.
We compare the characteristic times $T_{\rm R},T_{\rm D}$ in \Cref{def:characteristic-time} as $S\to\infty$.

\textbf{Case 1: Dueling Feedback is More Informative.}

Let
\[
    \bm x_1=\tfrac12\bm e_1,\qquad
    \bm x_2=(\tfrac12-\tfrac1S)\bm e_1+\tfrac12\bm e_2,\qquad
    \bm x_3=(\tfrac12-\tfrac1S)\bm e_1-\tfrac12\bm e_2.
\]
Arm $1$ is optimal, and both linear-score gaps equal $1$.

\emph{Reward-only.}
All reward curvatures are $\Theta(e^{-S/2})$.
Equally querying arms $2$ and $3$ attains comparison variance $O(B_\cc e^{S/2})$; bounded second-coordinate information makes this order necessary. Thus
\[
    T_{\rm R}=\Theta\!\left(B_\cc e^{S/2}\right).
\]

\emph{Dueling-only.}
Both best-arm comparisons have predictor $1$ and constant curvature.
Equal sampling attains comparison variance of order $B_\dd$, with a matching lower bound from bounded second-coordinate information. Hence
\[
    T_{\rm D}=\Theta(B_\dd),
    \qquad
    \frac{T_{\rm R}}{T_{\rm D}}
    =\Theta\!\left(\frac{B_\cc}{B_\dd}e^{S/2}\right).
\]

\textbf{Case 2: Reward Feedback is More Informative.}

Let
\[
    \bm x_1=\tfrac12\bm e_1,\qquad
    \bm x_2=-\tfrac12\bm e_1+\tfrac12\bm e_2,\qquad
    \bm x_3=-\tfrac12\bm e_1-\tfrac12\bm e_2.
\]
Arm $1$ remains optimal, and both linear-score gaps equal $S$.

\emph{Reward-only.}
All reward curvatures equal $\sigma'(S/2)=\Theta(e^{-S/2})$.
Equally querying arms $2$ and $3$ is optimal and estimates both coordinates, yielding
\[
    T_{\rm R}=\Theta\!\left(\frac{B_\cc e^{S/2}}{S^2}\right).
\]

\emph{Dueling-only.}
Only pairs $(1,2)$ and $(1,3)$ provide first-coordinate information, with curvature $\sigma'(S)=\Theta(e^{-S})$.
This limits every design; equally sampling the two pairs attains the same order. Hence
\[
    T_{\rm D}=\Theta\!\left(\frac{B_\dd e^S}{S^2}\right),
    \qquad
    \frac{T_{\rm R}}{T_{\rm D}}
    =\Theta\!\left(\frac{B_\cc}{B_\dd}e^{-S/2}\right).
\]
Since $B_\cc\asymp B_\dd$, the two geometries favor opposite modalities.

\end{document}